%% file: main_tmlr.tex
\documentclass[10pt]{article} 
\usepackage[preprint]{tmlr}

\input{math_commands.tex}

\usepackage{amsmath,amssymb,amsthm,mathtools}

\usepackage{url}
\usepackage{hyperref}
\usepackage{xcolor}
\definecolor{darkblue}{rgb}{0.0,0.0,0.55}
\hypersetup{
    colorlinks=true,
    linkcolor=darkblue,
    citecolor=darkblue,
    urlcolor=darkblue
}

\theoremstyle{plain}
\newtheorem{theorem}{Theorem}[section]

\newtheorem{corollary}[theorem]{Corollary}

\theoremstyle{definition}

\newtheorem{assumption}[theorem]{Assumption}

\theoremstyle{remark}
\newtheorem{remark}[theorem]{Remark}
\newcommand\Tau{\mathcal{T}}

\usepackage{xcolor}

\definecolor{darkblue}{rgb}{0.0,0.0,0.55}

\title{Error Bounds for a Diffusion Model-Based Drift Estimator}

\author{
    \name Ioar Casado-Telletxea\thanks{This work was carried out while ICT was a visiting PhD student at the Manchester Centre for AI Fundamentals.} 
        \email icasado@bcamath.org \\
        \addr Basque Center for Applied Mathematics (BCAM) \\
        Bilbao, Spain 
    \AND
    \name Omar Rivasplata 
        \email omar.rivasplata@manchester.ac.uk \\
        \addr Centre for AI Fundamentals \& Department of Computer Science \\
        University of Manchester, UK 
    }

\def\month{MM}  
\def\year{YYYY} 
\def\openreview{\url{https://openreview.net/forum?id=XXXX}} 

\begin{document}

\maketitle

\begin{abstract}
Parameter estimation in stochastic differential equations is a classical statistical problem of much importance in many scientific fields. Recent work of \cite{costa2026drift} introduced a novel technique for estimating the drift when the diffusion parameter is known, using discrete samples from multiple trajectories.
Their method treats drift estimation as a denoising problem, and leverages tools from (conditional) score-matching diffusion models. Although their experiments showed promising results across different drift classes, the question of theoretical guarantees for their estimator was left unanswered. 
In this note, we address this gap by exploiting techniques from diffusion model theory. More concretely, we derive an explicit risk bound for the time-averaged mean-squared error of said drift estimator. Our bound decomposes the risk into the (i) Euler-Maruyama discretization, (ii) score/denoiser approximation, (iii) noise initialization, and (iv) sampling variance, revealing the trade-offs between the different hyperparameters and sources of error in the estimator. 
\end{abstract}

\tableofcontents

\section{Introduction}

Stochastic differential equations (SDEs) are mathematical models playing a fundamental role for modeling dynamical systems that are subject to deterministic trends and randomness, with important applications across physics, biology, finance, engineering, and the geosciences. Hence, reliable parameter estimation is essential both for scientific understanding and for prediction and control.
In particular, estimating the drift --that is, the deterministic part-- of an SDE from discretely observed sample paths is a classical problem in statistics \citep{kutoyants2004statistical}. 

Recently, \citet{costa2026drift} proposed a novel approach for this task. Instead of learning the drift directly via regression from the increments, which is the classical approach, they introduce a conditional denoising diffusion model for learning noisy one-step transitions and then recover the drift from this denoiser. Their approach connects drift estimation with the score-matching and denoising ideas that underlie modern diffusion models, and suggests a possible way to stabilize learning in regimes where direct regression may be difficult or unstable.

Although \citet{costa2026drift} show empirically that their method is competitive across drift classes, they do not provide a theoretical analysis of their estimator. This note complements their contribution by providing an explicit bound for the time-averaged mean-squared error of their drift estimator. The resulting decomposition separates the contribution of the Euler--Maruyama discretization, the score approximation error, noise initialization error, and the Monte Carlo estimation error. In this way, the result makes precise which parts of the method are responsible for the final estimation error, and how the main hyperparameters enter the trade-off. Remarkably, our bound captures the nonlinear relation between sampling frequency and mean-squared error observed experimentally by \citet{costa2026drift}.

The paper is structured as follows: in Section \ref{section:prelim} we introduce the prerequisites for our main results, namely, score-matching diffusion models (SDMs) and the SDM-based drift estimator of \cite{costa2026drift}. In Section \ref{section:main} we obtain error bounds for the time-averaged mean-squared error of the estimator and discuss their implications. Finally, the limitations of our approach and avenues for future work are discussed in Section \ref{section:conclusions}.

\subsection{Related work}

The statistical literature on drift estimation for diffusion processes is extensive; see \cite{kutoyants2004statistical} for general background and Section 2.1 of \cite{costa2026drift} for a recent overview. Many works formulate drift estimation as a nonparametric regression problem \citep{comte2007penalized, comte2020nonparametric, zhao2020state, denis2021ridge}, and more recent papers have considered neural-network-based regression \citep{oga2024drift, zhao2025drift}. 

In our setting, the training data are discrete observations from a finite set of i.i.d. trajectories from the SDE, observed at regular time intervals; see Subsection \ref{subsec:drift-sde} for the mathematical formulation of the problem. There is also a growing literature that aims to learn stochastic dynamics, but using uncorrelated samples from the marginals \citep{neklyudov2023action,guan2024identifying, lavenant2024toward}.

Our analysis of the diffusion-model-based estimator of \cite{costa2026drift} is informed by techniques from the convergence theory of SDMs such as \cite{chen2022sampling,mbackenote,gao2025wasserstein} or \cite{pfarr2026analyzing}. See Section 6 of \cite{tang2025score} for a survey on different approaches to diffusion model theory. We also note several recent diffusion model-based approaches that are adjacent to \cite{costa2026drift} in spirit. These include conditional diffusion methods for sampling SDE trajectories \citep{liu2025training, gao2025data}. However, these methods do not directly target drift estimation, and their convergence theory is largely unexplored.

\section{Preliminaries}\label{section:prelim}

\subsection{Introduction to Score-Matching Diffusion Models (SDMs)}
SDMs are a family of generative models that aim to generate new samples from an unknown distribution by learning to reverse a noise-injection procedure. Starting from the target distribution $X_0\sim p_0$, they gradually inject randomness until they end up with (almost) pure noise
\begin{equation}
X_0\rightarrow X_1\rightarrow...\rightarrow X_\Tau\sim p_{\Tau}\approx p_{\mathrm{noise}}.    
\end{equation}
This forward process is straightforward. The crux of diffusion models is to learn how to reverse it; that is, starting from $X_\Tau\sim p_{\mathrm{noise}}$, how to generate a sample from $p_0$ by inverting the forward process
\begin{equation}
    X_\Tau\rightarrow X_{\Tau-1}\rightarrow...\rightarrow X_0\sim p_0.  
\end{equation}
Formally, the forward process can be understood as a discretization of a (Markov) diffusion process, while the second step amounts to learning its corresponding time-reversed process. The problem can be thus expressed naturally in the language of stochastic differential equations (SDEs). The presentation below mainly follows \cite{tang2025score}.

Consider the continuous-time process $(X_\tau)_{\tau\in[0,\Tau]}$ described by the following SDE:
\begin{equation}\label{eq:forward-sde}
    dX_\tau=f(\tau,X_\tau)d\tau + g(\tau)dB_\tau,\quad X_0\sim p_0,
\end{equation}
where $f:\mathbb{R}_{+}\times\mathbb{R}^D\rightarrow \mathbb{R}^D$ and $g:\mathbb{R}_{+}\rightarrow \mathbb{R}_{+}$ are respectively the drift and noise parameters, and $(B_\tau)_{\tau\geq 0}$ is $D$-dimensional Brownian motion \citep{le2016brownian}. Note that, from a SDM perspective, $f$ and $g$ are chosen by the user: different diffusion models correspond to different choices of $f$ and $g$.

Now consider the time-reversal of the process (\ref{eq:forward-sde}), i.e., $(\widetilde{X}_\tau)_{\tau\in[0,\Tau]}:=(X_{\Tau-\tau})_{\tau\in[0,\Tau]}$. Under mild assumptions on $f$ and $g$, a theorem of \cite{anderson1982reverse} shows that $(\widetilde{X}_\tau)_{\tau\in[0,\Tau]}$ satisfies the SDE
\begin{equation}\label{eq:reverse-sde}
d\widetilde{X}_\tau=\left[-f(\Tau-\tau,\widetilde{X}_\tau) + g(\Tau-\tau)^2\nabla\log p_{\Tau-\tau}(\widetilde{X}_\tau)\right]d\tau + g(\Tau-\tau)dB_\tau,\quad \widetilde{X}_0\sim p_\Tau,
\end{equation}
where $p_\tau$ is the probability density of $X_\tau$ conditional on $X_0$.

It is worth noticing that the diffusion term in the SDE for $\widetilde{X}_\tau$ is given by the same function as the corresponding term in the SDE for $X_\tau$, namely $g$, which is evident by looking at equations (\ref{eq:reverse-sde}) and (\ref{eq:forward-sde}). By contrast, the drift term in (\ref{eq:reverse-sde}) uses not only the drift function $f$ in (\ref{eq:forward-sde}) but also additionally uses the diffusion function $g$ and the score function $\nabla\log p_\tau(X_\tau)$.

At this point, if we could sample from $p_\Tau$ and solve (\ref{eq:reverse-sde}), we would already be able to generate samples from $p_0$ by running the reverse process. However, several obstacles remain in our way.

First, except for very specific situations, $p_\Tau$ still depends on $p_0$ via $X_0$, so generally we are not able to sample from $p_\Tau$. A widely-used solution is to substitute $p_{\Tau}$ with some distribution $p_{\mathrm{noise}}$ from which it is easy to sample, and which is still close enough to $p_\Tau$.
Second, the score function, $\nabla\log p_\tau(X_\tau)$, is unknown. In fact, the core idea behind SDMs is to estimate $\nabla\log p_\tau(X_\tau)$ by training a neural network $s^\theta_\tau(x)$ using score-matching \citep{hyvarinen2005estimation}. See Section 4.2 on \cite{lai2025principles} or Section 4 in \cite{tang2025score} for a discussion of several techniques.
Finally, after these approximations, we can solve (\ref{eq:reverse-sde}) using SDE discretization schemes such as the Euler-Maruyama method \citep{kloeden1992}.

\subsubsection{Example: Variance preserving (VP) SDEs}\label{subsection:VP-SDE}

VP SDEs were introduced by \cite{songscore} as the continuous limit of denoising diffusion probabilistic models (DDPMs) \citep{ho2020denoising}. DDPMs use $N$ noise scales $\beta_1<\cdots<\beta_N$ and run the (forward) process
\begin{equation}
    x_i=\sqrt{1-\beta_i}x_{i-1} + \sqrt{\beta_i}z_i,\quad 1\leq i\leq N,
\end{equation}
where $z_i$ are i.i.d. samples from $\mathcal{N}(0,I_D)$.
It can be shown (see Section 3(c) of \cite{tang2025score} for the details) that this scheme corresponds to the continuous-time SDE
\begin{equation}\label{eq:VP-SDE-forward}
    dX_\tau=-\frac{1}{2}\beta(\tau)X_\tau d\tau + \sqrt{\beta(\tau)}dB_\tau,\quad 0\leq \tau\leq \Tau,
\end{equation}
which is a particular case of (\ref{eq:forward-sde}) with the choices $f(\tau,x)=-\frac{1}{2}\beta(\tau)x$ and $g(\tau)=\sqrt{\beta(\tau)}$. By solving (\ref{eq:VP-SDE-forward}) we get the distribution 
\begin{equation}\label{eq:forward-kernel}
    p_\tau(\cdot|X_0=x)=\mathcal{N}\left(e^{-\frac{1}{2}\int_0^\tau\beta(s)ds}x,\left(1-e^{-\int_0^\tau\beta(s)ds}\right)I_D \right).
\end{equation}
For common choices of $\beta(\tau)$ and $\Tau$ large enough, it is thus reasonable to choose $p_{\mathrm{noise}}=\mathcal{N}(0,I_D)$, which finally results in the following approximate reverse SDE:
\begin{equation}\label{eq:approx-sde-VP}
d\widetilde{X}_\tau=\left[\frac{1}{2}\beta(\Tau-\tau)\widetilde{X}_\tau + \beta(\Tau-\tau) s^{\theta}_{\Tau-\tau}(\widetilde{X}_\tau)\right]d\tau + \sqrt{\beta(\Tau-\tau)}dB_\tau,\quad \widetilde{X}_0\sim \mathcal{N}(0,I_D).
\end{equation}
While \cite{costa2026drift} used the specific schedule $\beta(\tau)=\gamma_0+\tau(\gamma_1-\gamma_0)$, for some $\gamma_1>\gamma_0>0$, the results we present here are valid for arbitrary $\beta(\tau)$.

\subsection{Estimating the drift of an SDE}\label{subsec:drift-sde}

Suppose we are interested in estimating the (time-homogeneous) drift function, $\mu:\mathbb{R}^D\rightarrow \mathbb{R}^D$, of the following continuous-time SDE:
\begin{equation}\label{eq:target-SDE}
    dY_t = \mu(Y_t)dt + \sigma dB_t,\quad t\in[0,T+\Delta],
\end{equation}
where the diffusion parameter $\sigma>0$ is known. For this purpose, we are given a dataset $\mathcal{D}_\text{obs}$ consisting of i.i.d. trajectories from (\ref{eq:target-SDE}), say $I$ of them, each trajectory with $J$ observations obtained at times $t_j := j\Delta$ for $j=0,\ldots,J$, where $\Delta=T/J$ is the frequency at which we gather observations. That is,
\begin{equation}\label{eq:training-data}
\mathcal{D}_\text{obs} = \left\{Y_{t_0}^{(i)},\ldots, Y_{t_J}^{(i)}\right\},\quad i=1,\ldots, I.  
\end{equation}

\begin{remark}
    We define the SDE (\ref{eq:target-SDE}) on $[0,T+\Delta]$ for notational convenience and to make the bound in Theorem \ref{theorem:main} cleaner. The SDE can be defined in $[0,T]$ at the cost of dividing by $T-\Delta$ and integrating up to $T-\Delta$ in the time-averaged MSE error (\ref{eq:MSE-error}).
\end{remark}

The classical approach to the drift estimation problem is regression over the increments. Let $Z_{t_j}:=Y_{t_{j}}-Y_{t_{j-1}}$, for $j=1,\ldots, J$. The solution to (\ref{eq:target-SDE}) on $[t_{j-1},t_j)$ is of the form
\begin{equation}\label{eq:SDE-interval}
    Z_{t_j}=\int_{t_{j-1}}^{t_j}\mu(Y_s)ds + \sigma\sqrt{\Delta}\omega,\quad\omega\sim\mathcal{N}(0,I_D).
\end{equation}
Observe that, due to the Markov property of $(Y_t)_{t\in[0,T]}$ and the time-homogeneity of $\mu$, the random variables $Z_{t_j}\,|\,Y_{t_{j-1}}$ are identically distributed for every $j=1,\ldots,J$.
Now, given the frequency $\Delta$, an Euler--Maruyama (EM) approximation to (\ref{eq:SDE-interval}) yields
\begin{equation}\label{eq:EM-approx}
    Z_{t_j}\approx\widehat{Z}_{t_j}:=\mu(Y_{t_{j-1}})\Delta + \sigma\sqrt{\Delta}\omega,\quad \omega\sim\mathcal{N}(0,I_D),
\end{equation}
with $\widehat{Z}_{t_{j}}\,|\,(Y_{t_{j-1}}=y)\sim \mathcal{N}\left(\mu(y)\Delta,\sigma^2\Delta I_D\right)$, and $\mathbb{E}\left[\widehat{Z}_{t_j}\,|\,Y_{t_{j-1}}=y\right]=\mu(y)\Delta$.  This suggests learning $\mu$ by regression:
\begin{equation}
    \underset{\theta}{\min}\sum_{i\in I,j\in J}\left\|D_\theta\bigl(Y^{(i)}_{t_{j-1}}\bigr)-Z_{t_j}^{(i)}\right\|^2_2,
\end{equation}
where $D_\theta$ is a fixed class of functions. However, \cite{costa2026drift} argue that this approach can suffer from high variance and the curse of dimensionality. Inspired by approaches that regularize regression problems with input noise, they propose using a conditional diffusion model to estimate the SDE drift.

\subsubsection{Drift estimation with diffusion models}
Their key observation is that, since the increments $Z_{t_j}$ can be approximated as noisy versions of the signal $\mu(Y_{t_{j-1}})\Delta$, one can treat drift estimation as a denoising problem. Concretely, \cite{costa2026drift} apply the VP SDE framework of Section \ref{subsection:VP-SDE} to $p_0$, the density of $Z_{t_j}\,|\, Y_{t_{j-1}}$, and then extract the drift from the learned denoiser. As mentioned above, the law of $Z_{t_j}\,|\, Y_{t_{j-1}}$ is the same for every $t_j$, so we just write $Z\,|\, Y$ from now on.

We now apply the VP forward process to the increments $Z$, treating $Y$ as a fixed conditioning variable. That is, $Z$ plays the role of $X_0$ in (\ref{eq:VP-SDE-forward}). For a given noise level $\tau > 0$, the forward kernel produces the noisy version
\begin{equation}\label{eq:forward-noising}
    X_\tau = \alpha_\tau Z + \sigma_\tau \varepsilon, \qquad \varepsilon \sim \mathcal{N}(0, I_D),
\end{equation}
where $\alpha_\tau = e^{-\frac{1}{2}\int_0^\tau\beta(s)\,ds}$ and $\sigma_\tau^2 = 1 - \alpha_\tau^2$ are exactly as in (\ref{eq:forward-kernel}). The conditional marginal at noise level $\tau$ given $Y = y$ is thus
\begin{equation}
    p_\tau(x_\tau \mid Y= y) = \int_{\R^D} \mathcal{N}\bigl(x_\tau;\; \alpha_\tau z,\; \sigma_\tau^2 I_D\bigr)\, p_0(z \mid y)\, dz.
\end{equation}

A neural denoiser $D_\theta(\tau, x_\tau, y)$ is then trained to predict the increment $Z$ from the noisy version $X_\tau$ and the conditioning state $Y$, by minimizing the conditional denoising loss
\begin{equation}\label{eq:denoising-loss}
    \mathcal{L}(\theta) = \mathbb{E}_{\tau, Y, Z, X_\tau}\left\|D_\theta(\tau, X_\tau, Y) - Z\right\|_2^2,
\end{equation}
where the expectation is over $\tau \sim \mathcal{U}[\varepsilon, \Tau]$, $(Y, Z)$ drawn from the pooled training pairs, and $X_\tau$ sampled via equation (\ref{eq:forward-noising}).  The population minimizer of this conditional denoising loss is $\mathbb{E}_{p_0}[Z \mid X_\tau, Y]$.
 
\begin{remark}[Connection to score-matching]
The loss (\ref{eq:denoising-loss}) is equivalent to the conditional score-matching loss up to reparametrization.  Indeed, the conditional score of the VP SDE is
\[
    \nabla_{X_\tau} \log p_\tau(X_\tau \mid Y) = -\frac{X_\tau}{\sigma_\tau^2} + \frac{\alpha_\tau}{\sigma_\tau^2}\mathbb{E}_{p_0}[Z \mid X_\tau , Y],
\]
so learning the conditional score is equivalent to learning the denoiser $\mathbb{E}_{p_0}[Z \mid X_\tau, Y]$.
\end{remark}

The trained denoiser $D_{\theta^*}(\tau, x_\tau, y) \approx \mathbb{E}_{p_0}[Z \mid X_\tau = x_\tau, Y = y]$ approximates the posterior mean of the true increments.  To obtain a closed-form relation between the denoiser and the drift $\mu(y)$, \cite{costa2026drift} exploit the EM approximation (\ref{eq:EM-approx}):
\[
    Z \mid (Y = y) \;\approx\; \hat{p}_{\mathrm{data}}(\cdot\mid y) := \mathcal{N}\bigl(\mu(y)\Delta,\; \sigma^2\Delta\, I_D\bigr).
\]
Under this Gaussian approximation, one obtains the conditional mean
\begin{equation}\label{eq:em-posterior}
    \mathbb{E}_{\hat{p}_{\mathrm{data}}}[Z \mid X_\tau = x_\tau,\, Y = y]
    = \frac{\sigma_\tau^2 \Delta}{\sigma_\tau^2 + \alpha_\tau^2\sigma^2\Delta}\left(\mu(y) + \frac{\alpha_\tau\sigma^2}{\sigma_\tau^2}\,x_\tau\right).
\end{equation}
Finally, solving equation (\ref{eq:em-posterior}) for $\mu(y)$ yields
\begin{equation}\label{eq:drift-formula}
    \mu(y) = a(\tau)\,x_\tau + b_\Delta(\tau)\,\mathbb{E}_{\hat{p}_{\mathrm{data}}}[Z \mid X_\tau = x_\tau,\, Y = y],
\end{equation}
where
\begin{equation}\label{eq:ab-coeffs}
    a(\tau) := -\frac{\alpha_\tau\sigma^2}{\sigma_\tau^2}, \qquad
    b_\Delta(\tau) := \frac{\sigma_\tau^2 + \alpha_\tau^2\sigma^2\Delta}{\sigma_\tau^2 \Delta}.
\end{equation}
Observe that the identity (\ref{eq:drift-formula}) is only exact under the EM Gaussian approximation. This distinction is central to the error analysis in Theorem \ref{theorem:main}.
 
Replacing the conditional expectation in (\ref{eq:drift-formula}) by the learned denoiser $D_{\theta^*}$, we obtain the single-sample plug-in drift estimator
\begin{equation}\label{eq:single-sample}
    \hat\mu(\tau, x_\tau, y) := a(\tau)\,x_\tau + b_\Delta(\tau)\,D_{\theta^*}(\tau, x_\tau, y).
\end{equation}
In practice, the estimator is averaged over $K$ i.i.d. samples $G^{(1)}, \ldots, G^{(K)} \sim \mathcal{N}(0,I_D)$:
\begin{equation}\label{eq:drift-estimator}
    \bar\mu_K(\tau, y) := \frac{1}{K}\sum_{k=1}^K \hat\mu\bigl(\tau,\, G^{(k)},\, y\bigr).
\end{equation}
Hence the training data is only used to learn the denoiser $D_{\theta^*}$, while at test time the estimator can, in principle, be evaluated at any state $y$ for which the denoiser is defined.
 
Our theoretical guarantees below are stated for the time-averaged MSE error 
\begin{equation}\label{eq:MSE-error}
    \frac{1}{T}\int_0^T\mathbb{E}\bigl\|\mu(Y_{t}) - \bar\mu_K(\tau, Y_{t})\bigr\|_2^2,
\end{equation}

where the expectation is taken with respect to $Y_t$ and the samples $G^{(k)}\sim \mathcal{N}(0,I_D)$ in $\bar\mu_K(\tau, y)$. This is the natural error metric for drift estimation; see \citep{comte2020nonparametric, oga2024drift, costa2026drift}.

\section{Main results}\label{section:main}

We now introduce assumptions on the SDE in (\ref{eq:target-SDE}) and the learned denoiser, $D_{\theta^*}$, that will allow us to obtain our main result, a time-averaged mean-squared error bound for the estimator (\ref{eq:drift-estimator}).

\begin{assumption}[Lipschitz drift]\label{ass:Lipschitz-drift}
    There is a constant $L>0$ such that, for every $x,y\in\mathbb{R}^D$,
    \[
    \|\mu(x)-\mu(y)\|_2\leq L\|x-y\|_2.
    \]
\end{assumption}

\begin{assumption}[Finite second moment]\label{ass:bounded-second-moment}
Let $(Y_t)_{t\in[0,T+\Delta]}$ be a solution to (\ref{eq:target-SDE}). Then for every $t\in[0,T+\Delta]$, we have
\[
\mathfrak{m}_\mu(t):= \mathbb{E}\|\mu(Y_t)\|^2_2<\infty.
\]
\end{assumption}

\begin{assumption}[$L^2$-accurate score-matching]\label{ass:score-accuracy}
    For any $\tau>0$ there is $\varepsilon_{\mathrm{score}}(\tau)>0$ such that 
    \[
    \frac{1}{T}\int_0^T\mathbb{E}\left\|D_{\theta^*}(\tau,X_\tau,Y_t)-Z\right\|^2_2dt\leq \varepsilon_{\mathrm{score}}(\tau).
    \]
\end{assumption}

\begin{assumption}[Lipschitz denoiser]\label{ass:lipschitz-denoiser}
    For any $\tau>0$, there is $L_{D}(\tau)>0$, such that for every $x,x',y\in\mathbb{R}^D$,
    \[
    \|D_{\theta^*}(\tau,x,y)-D_{\theta^*}(\tau,x',y)\|_2\leq L_{D}(\tau)\|x-x'\|_2.
    \]
\end{assumption}

Some comments on our assumptions are pertinent. Assumption \ref{ass:Lipschitz-drift} is mild in the sense that for (\ref{eq:target-SDE}) to have strong solutions, Lipschitzness of $\mu$ is the usual sufficient assumption in SDE theory. At least local Lipschitzness is necessary. See the discussion in Section 5.6 of \cite{van2007stochastic}. 

In the same way, under Assumption \ref{ass:Lipschitz-drift} and standard linear growth conditions ensuring existence and uniqueness of strong solutions to (\ref{eq:target-SDE}), Assumption \ref{ass:bounded-second-moment} follows from Gronwall's inequality \citep{oksendal2003stochastic}. Observe that Assumptions \ref{ass:Lipschitz-drift} and \ref{ass:bounded-second-moment} together imply $\int_0^{T+\Delta}\mathfrak{m}_\mu(t)dt<\infty.$

Assumption \ref{ass:score-accuracy} imposes oracle denoiser accuracy, which is standard in diffusion model theory \citep{gao2025wasserstein, tang2025score}. On the one hand, this avoids the difficult question of obtaining approximation rates for score-matching under specific architectures. At the same time, it also serves as a plug-in assumption that can accommodate any of those rates, depending on the specific architecture chosen by the user. 

Finally, Assumption \ref{ass:lipschitz-denoiser} holds for many standard neural network architectures, such as any finite MLP with globally Lipschitz activations and fixed weights. In particular, it is satisfied by the architecture described in Appendix B of \cite{costa2026drift}, since the $x$-dependence enters the denoiser through the convolutional and conditioning parts, which are both a composition of linear layers and Lipschitz activations. However, the resulting constant $L_D(\tau)$ is not explicitly controlled during training, and in practice may depend on the dimension $D$. Regularity assumptions such as \ref{ass:lipschitz-denoiser} are also common when analyzing diffusion models \citep{chen2022sampling, gao2025wasserstein, pfarr2026analyzing}.

We are now ready to state our main result, which decomposes the average MSE of the drift estimator into four sources of error.

\subsection{The error bound}

\begin{theorem}[Main Theorem]\label{theorem:main} 
Let us have a training dataset $\mathcal{D}_{\text{obs}}$ as in (\ref{eq:training-data}) for certain frequency $\Delta>0$. 
Fix the noise level $\tau>0$ and the number of noisy samples $K$. Under assumptions \ref{ass:Lipschitz-drift}, \ref{ass:bounded-second-moment}, \ref{ass:score-accuracy}, and \ref{ass:lipschitz-denoiser}, we have
\begin{equation}
\begin{aligned}
\frac{1}{T}\int_0^{T}\E\bigl\|\mu(Y_t)-\bar\mu_K(\tau,Y_t)\bigr\|_2^2dt
\leq 4\bigg[
&\underbrace{L^2\left(\frac{\Delta^2}{T}\int_0^{T+\Delta}\mathfrak{m}_\mu(s)ds+\sigma^2 D\Delta\right)}_{\text{(I) discretization error}}
+\underbrace{b_\Delta(\tau)^2\varepsilon_{\mathrm{score}}(\tau)}_{\text{(II) score approximation error}}\\
&+\underbrace{L_{\hat\mu}(\tau)^2
\left(2\alpha_\tau^2\left(\frac{\Delta^2}{T}\int_0^{T+\Delta}\mathfrak{m}_\mu(s)ds+\sigma^2 D\Delta\right)
+(1-\sigma_\tau)^2D
\right)}_{\text{(III) noise initialization error}}\\
&+\underbrace{\frac{D}{K}L_{\hat\mu}(\tau)^2\bigg],}_{\text{(IV) Monte Carlo sampling error}}
\end{aligned}
\end{equation}
where $L_{\hat\mu}(\tau):=|a(\tau)|+b_\Delta(\tau)L_D(\tau).$
\end{theorem}

\begin{proof}

The basic idea, common in the diffusion model literature, is to decompose the error $\|\mu(y)-\bar{\mu}_K(\tau,y)\|_2^2$ into different sources and then bound each term separately. We first introduce the key intermediate quantities.

For any $y\in\R^D$, we define the conditional mean of the single-sample estimator
\begin{equation}\label{eq:mu-tilde}
\widetilde\mu(\tau,y) := \E_{X_\tau\mid y}\bigl[\hat\mu(\tau,X_\tau,y)\bigr].
\end{equation}
We also introduce the average single-sample estimator for the case where $X_\tau\sim \mathcal{N}(0,I_D)$. This will capture the noise initialization error: 

\begin{equation}\label{eq:mu-tilde-normal}
\widetilde\mu_\mathcal{N}(\tau,y) := \E_{X_\tau\sim \mathcal{N}(0,I_D)}\bigl[\hat\mu(\tau,X_\tau,y)\bigr].
\end{equation}

Finally, we define 
\begin{equation}\label{eq:mu-Delta}
\mu_\Delta(y) := \frac{1}{\Delta}\,\E\bigl[Y_{t+\Delta}-Y_{t}\;\big|\;Y_{t}=y\bigr]
= \frac{
1}{\Delta}\int_0^\Delta \E\bigl[\mu(Y_{t+h})\;\big|\;Y_{t}=y\bigr]\,dh.
\end{equation}
This is the time-averaged expected drift over one step starting from $y$: it only depends on the SDE and the observation frequency $\Delta$. Note that the second equality arises by solving the SDE in $[t,t+\Delta)$.

For each $y\in\mathbb{R}^D$, consider the following decomposition of the error term:
\begin{equation}\label{eq:decomposition}
\mu(y)-\bar\mu_K(\tau,y)
= \underbrace{\bigl[\mu(y)-\mu_\Delta(y)\bigr]}_{\text{EM discretization}}
+ \underbrace{\bigl[\mu_\Delta(y)-\widetilde\mu(\tau,y)\bigr]}_{\text{score approximation}}
+ \underbrace{\bigl[\widetilde\mu(\tau,y)-\widetilde\mu_\mathcal{N}(\tau,y)\bigr]}_{\text{noise initialization}}
+ \underbrace{\bigl[\widetilde\mu_\mathcal{N}(\tau,y)-\bar\mu_K(\tau,y)\bigr]}_{\text{sampling}}.
\end{equation}

We now proceed to control each term separately.
 
\bigskip

\noindent\textbf{Term (I): EM discretization error.}\;
By definition:
\begin{equation}
\mu(Y_t)-\mu_\Delta(Y_t) = \frac{1}{\Delta}\int_0^\Delta\Bigl(\mu(Y_t)-\E\bigl[\mu(Y_{t+h})\;\big|\;Y_{t}\bigr]\Bigr)\,dh.
\end{equation}
Taking the squared norm, applying Jensen's inequality and taking expectations, we have
\begin{equation}\label{eq:lip-step}
\begin{aligned}
\E\bigl\|\mu(Y_t)-\mu_\Delta(Y_t)\bigr\|_2^2
&\leq \frac{1}{\Delta}\int_0^\Delta\E\left[\left\|\mu(Y_{t})-\mu(Y_{t+h})\right\|_2^2\right]\,dh\\
&\leq  \frac{L^2}{\Delta}\int_0^\Delta\E\left[\|Y_{t+h}-Y_{t}\|_2^2\right]\,dh.
\end{aligned}
\end{equation}

The first inequality is due to Fubini and the tower property of conditional expectations, while in the second we simply applied Assumption \ref{ass:Lipschitz-drift}.

Furthermore, for $h\in[0,\Delta]$, we have
\begin{equation}\label{eq:increment}
Y_{t+h}-Y_{t} = \int_{t}^{t+h}\mu(Y_s)\,ds + \sigma(B_{t+h}-B_{t}),
\end{equation}
which is the solution to the SDE on $[t,t+h)$. This implies
\begin{equation}\label{eq:increment-bound}
\begin{aligned}
\E\bigl[\|Y_{t+h}-Y_{t}\|_2^2\bigr]
&\leq 2\,\E\left[\Bigl\|\int_{t}^{t+h}\mu(Y_s)\,ds\Bigr\|_2^2\right] + 2\sigma^2\,\E\|B_{t+h}-B_{t}\|_2^2 \\
&\leq 2h\int_{t}^{t+h}\E\bigl[\|\mu(Y_s)\|_2^2\bigr]\,ds + 2\sigma^2 Dh \\
&= 2h\int_t^{t+h}\mathfrak{m}_\mu(s)ds + 2\sigma^2 Dh.
\end{aligned}
\end{equation}
Observe that in the first inequality we used $(a+b)^2\leq 2a^2+2b^2$. 
For the second inequality we used the Cauchy--Schwarz inequality for the inner product $\langle u,v\rangle:=\int_{t}^{t+h}u(s)\cdot v(s)\,ds$ with $u(s)=1$ and $v(s)=\mu(Y_s)$, along with the fact that $B_{t+h}-B_{t}$ is a centered Gaussian with variance $h$. For the last bound we simply invoke Assumption \ref{ass:bounded-second-moment} to guarantee that the integral is finite.
 
Finally, substituting (\ref{eq:increment-bound}) back into~(\ref{eq:lip-step}) and integrating:
\begin{equation}\label{eq:term-I-pointwise}
\begin{aligned}
\frac{1}{T}\int_0^{T}\E\bigl\|\mu(Y_t)-\mu_\Delta(Y_t)\bigr\|_2^2dt
&\leq \frac{L^2}{\Delta T}\int_0^{T}\int_0^\Delta \left(2h\int_t^{t+h}\mathfrak{m}_\mu(s)ds + 2\sigma^2 Dh\right)\,dh\,dt\\
&\leq L^2\left(\frac{\Delta^2}{T}\int_0^{T+\Delta}\mathfrak{m}_\mu(s)ds+\sigma^2 D\Delta\right),
\end{aligned}
\end{equation}
where we used Fubini's theorem to swap the integrals.
\bigskip

\noindent\textbf{Term (II): score approximation bias.}\;
Recall that $X_\tau = \alpha_\tau Z + \sigma_\tau \varepsilon$, with $\varepsilon \sim \mathcal \mathcal{N}(0,I_D)$. We introduce the \emph{oracle estimator}
\begin{equation}\label{eq:oracle-estimator}
\widehat\mu^{\mathrm{orc}}(\tau,x,y)
:= a(\tau)x + b_\Delta(\tau)\,\E_{p_0}[Z\mid X_\tau=x,\,Y=y],
\end{equation}
which replaces the learned denoiser $D_{\theta^*}$ by the true conditional expectation. A direct computation using the identity $a(\tau)\alpha_\tau+b_\Delta(\tau)=1/\Delta$ yields
\begin{equation}\label{eq:oracle-identity-term-II}
\mu_\Delta(y)
=
\E\!\left[\widehat\mu^{\mathrm{orc}}(\tau,X_\tau,y)\,\big|\,Y=y\right].
\end{equation}

Using (\ref{eq:oracle-identity-term-II}), the gap between $\mu_\Delta$ and $\widetilde\mu$ becomes a conditional expectation of the denoiser error:
\begin{align}
\mu_\Delta(y)-\widetilde\mu(\tau,y)
&=
\E\left[\widehat\mu^{\mathrm{orc}}(\tau,X_\tau,y)-\widehat\mu(\tau,X_\tau,y)\,\big|\,Y=y\right] \notag \\
&=b_\Delta(\tau)\,\E\!\left[\E_{p_0}[Z\mid X_\tau,Y]-D_{\theta^*}(\tau,X_\tau,Y)\,\big|\,Y=y\right].
\label{eq:term-II-gap}
\end{align}
Taking squared norms, applying Jensen's inequality, and then taking expectation with respect to $Y$ gives
\begin{equation}\label{eq:term-II-after-jensen}
\E\|\mu_\Delta(Y)-\widetilde\mu(\tau,Y)\|_2^2
\le
b_\Delta(\tau)^2\,\E\!\left\|\E_{p_0}[Z\mid X_\tau,Y]-D_{\theta^*}(\tau,X_\tau,Y)\right\|_2^2.
\end{equation}
Since $\E_{p_0}[Z\mid X_\tau,Y]$ is the $L^2$-optimal predictor of $Z$ among all $\sigma(X_\tau,Y)$-measurable functions, we have
\begin{equation}\label{eq:term-II-optimality}
\E\!\left\|\E_{p_0}[Z\mid X_\tau,Y]-D_{\theta^*}(\tau,X_\tau,Y)\right\|_2^2
\le
\E\|Z-D_{\theta^*}(\tau,X_\tau,Y)\|_2^2.
\end{equation}
Combining (\ref{eq:term-II-after-jensen}) and (\ref{eq:term-II-optimality}) with Assumption~\ref{ass:score-accuracy} yields
\begin{equation}\label{eq:term-II-final}
\frac{1}{T}\int_0^T\E\|\mu_\Delta(Y_t)-\widetilde\mu(\tau,Y_t)\|_2^2dt
\le
b_\Delta(\tau)^2\,\varepsilon_{\mathrm{score}}(\tau).
\end{equation}

\bigskip

\noindent\textbf{Term (III): Noise initialization error.}\;
Recall the definitions
\begin{equation}\label{eq:term-III-definitions}
\widetilde\mu(\tau,y):=\E_{X_\tau\sim p_\tau(\cdot\mid y)}[\widehat\mu(\tau,X_\tau,y)],
\qquad
\widetilde\mu_\mathcal{N}(\tau,y):=\E_{G\sim \mathcal{N}(0,I_D)}[\widehat\mu(\tau,G,y)],
\end{equation}
where $\widehat\mu(\tau,x,y)=a(\tau)x+b_\Delta(\tau)D_{\theta^*}(\tau,x,y)$.

By Assumption~\ref{ass:lipschitz-denoiser}, the map $\widehat\mu(\tau,\cdot,y)$ is Lipschitz with constant $L_{\hat\mu}(\tau):=|a(\tau)|+b_\Delta(\tau)L_D(\tau)$. The Kantorovich--Rubinstein duality (Theorem 5.10 in \cite{villani2009optimal}) combined with the standard bound $W_1\le W_2$ then gives
\begin{equation}\label{eq:term-III-kr}
\|\widetilde\mu(\tau,y)-\widetilde\mu_\mathcal{N}(\tau,y)\|_2
\le
L_{\hat\mu}(\tau)\,W_2\!\left(p_\tau(\cdot\mid y),\mathcal{N}(0,I_D)\right),
\end{equation}

where $W_p$ is the $p-$Wasserstein distance. To control the $2-$Wasserstein distance on the right-hand side, we couple $X_\tau\sim p_\tau(\cdot\mid y)$ and $G\sim \mathcal{N}(0,I_D)$ by writing $X_\tau=\alpha_\tau Z+\sigma_\tau\varepsilon$ and $G=\varepsilon$ with $\varepsilon\sim \mathcal{N}(0,I_D)$, yielding
\begin{align}
W_2^2\!\left(p_\tau(\cdot\mid y),\mathcal{N}(0,I_D)\right)
&\le
\E\!\left[\|X_\tau-G\|_2^2\mid Y=y\right] \notag\\
&=
\alpha_\tau^2\,\E[\|Z\|_2^2\mid Y=y]+(1-\sigma_\tau)^2D.
\label{eq:term-III-coupling}
\end{align}
Combining equations (\ref{eq:term-III-kr}) and (\ref{eq:term-III-coupling}), squaring, and taking outer expectation gives
\begin{equation}\label{eq:term-III-pointwise}
\E\|\widetilde\mu(\tau,Y_t)-\widetilde\mu_N(\tau,Y_t)\|_2^2
\le
L_{\hat\mu}(\tau)^2\!\left(\alpha_\tau^2\,\E\|Z\|_2^2+(1-\sigma_\tau)^2D\right).
\end{equation}
Applying the increment bound (\ref{eq:increment-bound}) to control $\E\|Z\|_2^2$, and averaging over $t\in[0,T]$, we arrive at
\begin{equation}\label{eq:term-III-final}
\frac{1}{T}\int_0^{T}\E\|\widetilde\mu(\tau,Y_t)-\widetilde\mu_N(\tau,Y_t)\|_2^2\,dt
\le
L_{\hat\mu}(\tau)^2\!\left[2\alpha_\tau^2\!\left(\frac{\Delta^2}{T}\int_0^{T+\Delta}\mathfrak{m}_\mu(s)\,ds+\sigma^2 D\Delta\right)+(1-\sigma_\tau)^2D\right].
\end{equation}

\bigskip

\noindent\textbf{Term (IV): Monte Carlo error.}\;
Recall that
\begin{equation}\label{eq:term-IV-definitions}
\bar\mu_K(\tau,y)=\frac{1}{K}\sum_{k=1}^K\widehat\mu(\tau,G^{(k)},y),
\end{equation}
with $G^{(1)},\dots,G^{(K)}\overset{\mathrm{i.i.d.}}{\sim}\mathcal{N}(0,I_D)$.

Conditionally on $Y=y$, the random vectors $\widehat\mu(\tau,G^{(1)},y),\dots,\widehat\mu(\tau,G^{(K)},y)$ are i.i.d.\ with mean $\widetilde\mu_\mathcal{N}(\tau,y)$, so
\begin{equation}\label{eq:term-IV-variance}
\E_G\left[\|\widetilde\mu_\mathcal{N}(\tau,y)-\bar\mu_K(\tau,y)\|_2^2\right]
=
\frac{1}{K}\,\Var_G\!\bigl(\widehat\mu(\tau,G,y)\bigr).
\end{equation}
Since $\widehat\mu(\tau,\cdot,y)$ is $L_{\hat\mu}(\tau)$-Lipschitz, applying the Gaussian Poincar\'e inequality \citep{boucheron2013} componentwise yields
\begin{equation}\label{eq:term-IV-poincare}
\Var_G\!\bigl(\widehat\mu(\tau,G,y)\bigr)
\le
\E_G\|J_G\widehat\mu(\tau,G,y)\|_F^2
\le
D\,L_{\hat\mu}(\tau)^2,
\end{equation}
where $J_G\widehat\mu$ denotes the Jacobian with respect to $G$. Combining (\ref{eq:term-IV-variance}) and (\ref{eq:term-IV-poincare}) gives
\begin{equation}\label{eq:term-IV-final}
\E_G\!\left[\|\widetilde\mu_\mathcal{N}(\tau,y)-\bar\mu_K(\tau,y)\|_2^2\right]
\le
\frac{D}{K}\,L_{\hat\mu}(\tau)^2
=
\frac{D}{K}\bigl(|a(\tau)|+b_\Delta(\tau)L_D(\tau)\bigr)^2.
\end{equation}

The proof concludes by applying the inequality $(a+b+c+d)^2\le 4a^2+4b^2+4c^2+4d^2$ to the decomposition of the error and combining all the bounds.

\end{proof}

\subsection{Interpreting the bound}

Theorem \ref{theorem:main} makes explicit how the different hyperparameters in the estimator (\ref{eq:drift-estimator}) interact. We now describe each term separately and discuss their interactions.
\begin{itemize}
    \item \textbf{Term (I): discretization error.}
    This is the Euler--Maruyama bias coming from replacing the true increment law by its one-step Gaussian approximation. It is of order
    \(
    O(D\Delta),
    \)
    and, as expected, vanishes linearly as the observation grid becomes finer.

    \item \textbf{Term (II): score approximation error.}
    This term measures how accurately the learned denoiser approximates the expected increments. Since $b_\Delta(\tau)=O(1/\Delta)$ for fixed $\tau$, its contribution is of order
    \[
    O\!\left(\frac{\varepsilon_{\mathrm{score}}(\tau)}{\Delta^2}\right).
    \]
    Thus, for fixed noise level, the effect of denoiser error is upscaled as $\Delta\to 0$. This might come as a surprise at first, but it is natural since the drift contribution to the SDE vanishes faster than the noise as $\Delta\to 0$. Note that the nonlinear relation between the estimator performance and the frequency is also observed empirically in Appendix H of \cite{costa2026drift}.

    \item \textbf{Term (III): noise initialization error.}
    This term quantifies the mismatch between the actual forward marginal $p_\tau(\cdot\mid y)$ and the Gaussian initialization $\mathcal{N}(0,I_D)$. For most VP schedules, such as the one used by \cite{costa2026drift}, it vanishes exponentially fast when $\tau\to\infty$ due to the contribution of $\alpha_\tau$ and $1-\sigma_\tau$.

    \item \textbf{Term (IV): Monte Carlo sampling error.}
    This is the variance introduced by approximating the Gaussian expectation with $K$ i.i.d. samples. For fixed $\tau$ and $\Delta$, it decays as
    \(
    O(D/K),
    \) as expected.
\end{itemize}

In summary, decreasing $\Delta$ reduces the discretization error in Term (I) but can amplify the score approximation error in Term (II), and also in Terms (III), (IV), via $b_\Delta(\tau)$, so refining the observation grid only helps if the denoiser improves accordingly. Likewise, increasing $\tau$ suppresses the initialization error in Term (III) but may worsen the score approximation.

Overall, the bound might not be tight, especially because we are ignoring the dependence of $\epsilon_{\text{score}}(\tau)$ on $\Delta$, and the Lipschitz constant of $D_{\theta^*}$ might be large. Nonetheless, it is remarkable that it still captures the nonlinear dependency of the MSE on $\Delta$ observed by \cite{costa2026drift}.

Furthermore, Theorem \ref{theorem:main} also allows to choose the hyperparameters $\varepsilon_{\mathrm{score}}$, $K$ and $\tau$ depending on $\Delta$, in the same fashion as in Corollary 2.7 of \cite{pfarr2026analyzing}:

\begin{corollary}\label{corollary:tuning}
    Assume $\Delta<1$ and let $\alpha\in (0,1]$. Assume moreover that $L_D(\tau)$ remains uniformly bounded along the chosen noise levels. To obtain a bound of rate $\mathcal{O}(\Delta^\alpha)$ in Theorem \ref{theorem:main}, it is enough to choose 
    \begin{equation}
        \tau\asymp\alpha\log(1/\Delta),\quad K\asymp\Delta^{-(2+\alpha)},\quad\varepsilon_{\mathrm{score}}\asymp\Delta^{2+\alpha}.
    \end{equation}
\end{corollary}
\begin{proof}
    The result follows directly from Theorem \ref{theorem:main} after basic algebraic manipulations. Observe that we are implicitly assuming that we work with standard VP schedules such as the one used by \cite{costa2026drift}: $\beta(\tau)=\gamma_0+\tau(\gamma_1-\gamma_0)$, for some $\gamma_1>\gamma_0>0$.
\end{proof}
For example, for a bound of order $\mathcal{O}\left(\sqrt{\Delta}\right)$, we have $\tau\asymp \frac{1}{2}\log(1/\Delta),$ $K\asymp \Delta^{-5/2}$ and $\epsilon_{\text{score}}\asymp \Delta^{5/2}$.
We fixed $\Delta$ in Corollary \ref{corollary:tuning} because the sampling frequency is usually given in many problems, however similar relations could be obtained by fixing any other hyperparameter. 

As suggested before, among the requirements in Corollary \ref{corollary:tuning}, the denoiser accuracy, $\varepsilon_{\mathrm{score}} \asymp \Delta^{2+\alpha},$ is the most stringent. This dependence reflects the $\mathcal{O}(1/\Delta^2)$ amplification factor of $b_\Delta(\tau)^2$ inherent to the estimator. However, the requirement in Corollary \ref{corollary:tuning} is very likely overly pessimistic, because Theorem \ref{theorem:main} treats $\varepsilon_{\mathrm{score}}(\tau)$ as independent of $\Delta$, while in practice finer discretizations increase the training sample size and may yield increments closer to a Gaussian, making the denoising problem easier. Obtaining error bounds for denoiser estimation that take into account the discretization is thus an important avenue for future work.

\section{Conclusions and future work}\label{section:conclusions}

To our knowledge, Theorem \ref{theorem:main} gives the first theoretical guarantees for SDE parameter estimation via conditional diffusion models, in particular, for the drift estimator of \cite{costa2026drift}. The resulting error bound isolates four fundamental sources of error in this approach: discretization bias, denoiser or score-approximation error, noise initialization error and Monte Carlo sampling error.

There are several natural directions for future research. First, our results are based on an oracle assumption on the denoiser or score error. As mentioned above, end-to-end bounds that account for the statistical properties of the underlying neural architecture remain an important open problem. Second, the structure of Theorem \ref{theorem:main} and Corollary \ref{corollary:tuning} suggest that they may be used to obtain practical rules for hyperparameter tuning; we leave this direction for future work.

Furthermore, it would be valuable to extend the present framework beyond the basic setting considered here, and to develop comparable methods and theoretical guarantees for broader families of SDEs, such as equations with time-dependent drifts or unknown diffusion parameter. Finally, it seems possible to extend our theoretical guarantees to SDE sampling methods that also use conditional diffusion models such as \cite{liu2025training, gao2025data}.

\bibliographystyle{tmlr}
\bibliography{main}
\end{document}

%% file: math_commands.tex
\usepackage{amsmath,amsfonts,bm}

\def\eqref#1{equation~\ref{#1}}

\def\1{\bm{1}}

\DeclareMathAlphabet{\mathsfit}{\encodingdefault}{\sfdefault}{m}{sl}
\SetMathAlphabet{\mathsfit}{bold}{\encodingdefault}{\sfdefault}{bx}{n}

\newcommand{\E}{\mathbb{E}}

\newcommand{\R}{\mathbb{R}}

\newcommand{\Var}{\mathrm{Var}}

